\documentclass[english]{article}

   \PassOptionsToPackage{numbers, compress}{natbib}

    \usepackage[final]{neurips_2019}

\usepackage[utf8]{inputenc} 
\usepackage[T1]{fontenc}    
\usepackage[bookmarks=false, hidelinks]{hyperref}
\usepackage{url}            
\usepackage{booktabs}       
\usepackage{amsfonts}       
\usepackage{nicefrac}       
\usepackage{microtype}      
\usepackage{amsmath,amsthm,amssymb,bm,bbm}

\newcommand{\norm}[1]{\lVert{#1}\rVert}
\newtheorem{thm}{Theorem}
\newtheorem{lem}{Lemma}
\newtheorem{prop}{Proposition}
\newtheorem{rem}{Remark}

\newtheorem{cor}{Corollary}

\newcommand{\bmat}[1]{\begin{bmatrix}#1\end{bmatrix}}

\newcommand{\tp}{\mathsf{T}}

\renewcommand{\epsilon}{\varepsilon}

\DeclareMathOperator{\trace}{trace}

\DeclareMathOperator{\vect}{vec}
\DeclareMathOperator{\diag}{diag}
\DeclareMathOperator{\sym}{sym}
\DeclareMathOperator{\real}{real}
\DeclareMathOperator{\imag}{imag}

\newcommand{\R}{\mathbb{R}}						

\title{Characterizing the Exact Behaviors of \\Temporal Difference Learning Algorithms Using \\Markov Jump Linear System Theory}

\author{%
  Bin ~Hu, \,\,\, \,\,\,\,Usman Ahmed Syed \\
  Department of Electrical and Computer Engineering\\
  Coordinated Science Laboratory\\
  University of Illinois at Urbana-Champaign\\
}

\begin{document}

\maketitle

\begin{abstract}
In this paper, we provide a unified analysis of temporal difference learning algorithms with linear function approximators by exploiting their connections to Markov jump linear systems (MJLS). We tailor the MJLS theory developed in the control community to characterize the exact behaviors of the first and second order moments of a large family of temporal difference learning algorithms. For both the IID and Markov noise cases, we show that the evolution of some augmented versions of the mean and covariance matrix of the TD estimation error exactly follows the trajectory of a deterministic linear time-invariant (LTI) dynamical system. Applying the well-known LTI system theory, we obtain closed-form expressions for the mean and covariance matrix of the TD estimation error at any time step. 
We provide a tight matrix spectral radius condition to guarantee the convergence of the covariance matrix of the TD estimation error, and perform a perturbation analysis to characterize the dependence of the TD behaviors on learning rate. For the IID case, we provide an exact formula characterizing how the mean and covariance matrix of the TD estimation error converge to the steady state values.
For the Markov case,  we use our formulas to explain how the behaviors of TD learning algorithms are affected by learning rate and the underlying Markov chain. For both cases,  upper and lower bounds for the mean square TD error are provided. The mean square TD error is shown to converge linearly to an exact limit.  
\end{abstract}

\section{Introduction}

Reinforcement learning (RL) has shown great promise in solving sequential decision making tasks \cite{bertsekas1996neuro, sutton2018reinforcement}.
One important topic for RL is policy evaluation whose objective is to evaluate the value function of a given policy.
A large family of temporal difference (TD) learning methods including standard TD, GTD, TDC, GTD2, DTD, and ATD \cite{sutton1988learning, sutton2008convergent, sutton2009fast, Niao2019}  have been developed
to solve the policy evaluation problem. These TD learning algorithms have become important building blocks for RL algorithms. See \cite{dann2014policy} for a comprehensive survey.
Despite the popularity of TD learning, the  behaviors of these algorithms have not been fully understood from a theoretical viewpoint. 
The standard ODE technique \cite{TsiRoy1997, borkar2000ode, bhatnagar2012stochastic,kushner2003stochastic, borkar2009stochastic} can only be used to prove asymptotic convergence. Finite sample bounds are challenging to obtain and typically developed in a case-by-case manner.
  Recently, there have been intensive research activities focusing on establishing finite sample bounds for TD learning methods with linear function approximations under various assumptions. The IID noise case is covered in \cite{dalal2018finite, lakshminarayanan2018linear, liu2015finite}. In \cite{bhandari2018finite}, the analysis is extended for a Markov noise model but an extra projection step in the algorithm is required. Very recently, finite sample bounds for the TD method  (without the projection step) under the Markov assumption have been obtained in \cite{srikant2019finite}. The bounds in \cite{srikant2019finite} actually work for any TD learning algorithm that can be modeled by a linear stochastic approximation scheme.
It remains unclear how tight these bounds are (especially for the large learning rate region). 
To complement the existing analysis results and techniques, we propose a general unified analysis framework for TD learning algorithms by borrowing the Markov jump linear system (MJLS) theory \cite{costa2006} from the controls literature. 
Our approach is inspired by a recent research trend in applying control theory for analysis of optimization algorithms~\cite{Lessard2014, hu2017SGD, Bin2017COLT, BinHu2017, fazlyab2017analysis, van2017fastest, cyrus2018robust, sundararajan2017robust, hu2017control, hu2018dissipativity, fazlyab2017dynamical, aybat2018robust, lessard2019direct, han2019systematic, aybat2019universally, dhingra2018proximal, nelson2018integral}, and extends the jump system perspective for finite sum optimization methods in \cite{Bin2017COLT} to TD learning.

 Our key insight is that TD learning algorithms with linear function approximations are essentially just Markov jump linear systems. Notice that a MJLS is described by a linear state space model whose
state/input matrices are functions of a jump parameter sampled from a finite state Markov chain. Since 
the behaviors of MJLS have been well established in the controls field \cite{costa2006, feng1992stochastic, abou1995solution,  chizeck1986discrete, costa1993stability, ji1988controllability, ji1990jump, el1996robust, Fang2002, seiler2003bounded}, we can borrow the analysis tools there to analyze TD learning algorithms in a more unified manner. Our main contributions are summarized as follows.
\begin{enumerate}
\item We present a unified Markov jump linear system perspective on a large family of TD learning algorithms including TD, TDC, GTD, GTD2, ATD, and DTD. Specifically,
we make the key observation that these methods are just MJLS subject to some prescribed input.
\item By tailoring the existing MJLS theory, we show that the evolution of some augmented versions of the mean and covariance matrix of the estimation error in all above TD learning methods exactly follows the trajectory of a deterministic linear time-invariant (LTI) dynamical system for both the IID and Markov noise cases. As a result, we obtain unified closed-form formulas for the mean and covariance matrix of the TD estimation error at any time step. 
\item We provide a tight matrix spectral radius condition to guarantee the convergence of the covariance matrix of the TD estimation error  under the general Markov assumption. By using the matrix perturbation theory \cite{moro1997lidskii, kato2013perturbation, avrachenkov2013analytic, gonzalez2015laurent},  we perform a perturbation analysis to show the dependence of the behaviors of TD learning on learning rate in a more explicit manner.  For the IID case, we provide an exact formula characterizing how the mean and covariance matrix of  the TD estimation error converge to the steady state values at a linear rate. For the Markov case, we use our formulas to explain how the behaviors of  TD learning algorithms are affected by learning rate and the underlying Markov chain. For both cases, we have shown that the mean square error of TD learning converges linearly to an exact limit. In addition, upper and lower bounds for the mean square error of TD learning are simultaneously derived.
\end{enumerate}

 We view our proposed analysis
as a complement rather than a replacement for existing analysis techniques.  The existing analysis focuses on upper bounds for the TD estimation error.
Our closed-form formulas provide both upper and lower bounds for the mean square error of TD learning. Our analysis also sheds light on the exact limit of the final TD error and the related convergence rate.
%

\section{Background}

\subsection{Notation}
The set of $m$-dimensional real vectors is denoted as $\R^m$. 
The Kronecker product of two matrices $A$ and $B$ is denoted by $A \otimes B$. Notice $(A\otimes B)^T=A^T \otimes B^T$ and $(A\otimes B)(C\otimes D)=(AC)\otimes (BD)$ when the matrices have compatible dimensions.  
Let $\vect$ denote the standard vectorization operation that stacks the columns of a matrix into a vector. We have $\vect(AXB)=(B^\tp \otimes A) \vect(X)$. Let $\sym$ denote the symmetrization operation, i.e. $\sym(A)=\frac{A^\tp+A}{2}$.
Let $\diag(H_i)$ denote a matrix whose $(i,i)$-th block is $H_i$ and all other blocks are zero. Given $H_i$ for $i=1, \ldots, n$, we have
\begin{align*}
\diag(H_i)=\bmat{H_1 & \ldots & 0\\ \vdots & \ddots & \vdots \\ 0 & \ldots & H_n}.
\end{align*}
A square matrix is Schur stable if all its eigenvalues have magnitude strictly less than $1$. A square matrix is Hurwitz if all its eigenvalues have strictly negative real parts.
The spectral radius of a matrix $H$ is denoted as $\sigma(H)$. The eigenvalue with the largest magnitude of $H$ is denoted as $\lambda_{\max}(H)$ and the eigenvalue with the largest real part of $H$ is denoted as $\lambda_{\max\real}(H)$.

\subsection{Linear time-invariant systems}
\label{sec:LTI}

A linear time-invariant (LTI) system is typically governed by the following state-space model
\begin{align}
\label{eq:lin1}
x^{k+1}=\mathcal{H} x^k+\mathcal{G} u^k,
\end{align}
where $x^k\in \R^{n_x}$, $u^k\in \R^{n_u}$, $\mathcal{H}\in \R^{n_x\times n_x}$, and $\mathcal{G}\in \R^{n_x\times n_u}$.
The LTI system theory has been well  documented in standard control textbooks \cite{hespanha2009, chen1998linear}. Here we briefly review several useful results.

%
%
%
%
%
%

\begin{itemize}
\item \textbf{Closed-form formulas for $x^k$:} Given an initial condition $x^0$ and an input sequence $\{u^k\}$, the sequence $\{x^k\}$ can be determined using the following closed-form expression
\begin{align}
\label{eq:res1}
x^k= (\mathcal{H})^k x^0+\sum_{t=0}^{k-1} (\mathcal{H})^{k-1-t} \mathcal{G} u^t,
\end{align}
where $(\mathcal{H})^k$ stands for the $k$-th power of the matrix $\mathcal{H}$.

\item \textbf{Necessary and sufficient stability condition:}  When $\mathcal{H}$ is Schur stable,  we know $(\mathcal{H})^k x^0\rightarrow 0$ for any arbitrary $x^0$. 
 When $\sigma(\mathcal{H})\ge 1$,  there always exists $x^0$ such that $(\mathcal{H})^k x^0$ does not converge to $0$.  When $\sigma(\mathcal{H})> 1$,  there even exists $x^0$ such that $(\mathcal{H})^k x^0 \rightarrow \infty$. 
See Section 7.2 in \cite{hespanha2009} for a detailed discussion.  A well-known result in the controls literature is that the LTI system \eqref{eq:lin1} is stable if and only if $\mathcal{H}$ is Schur stable.

\item\textbf{Exact limit for $x^k$:}  If $\mathcal{H}$ is Schur stable and $u^k$ converges to a limit $u^\infty$, then $x^k$ will converge to an exact limit. This is formalized as follows.
\begin{prop}
\label{prop:ltilim}
Consider the LTI system \eqref{eq:lin1}. If $\sigma(\mathcal{H})<1$ and $\lim_{k\rightarrow \infty}u^k=u^\infty$, then 
$\lim_{k\rightarrow \infty} x^k$ exists and we have $x^\infty=\lim_{k\rightarrow \infty} x^k=(I-\mathcal{H})^{-1}\mathcal{G}u^\infty$.
\end{prop}
\item \textbf{Response for constant input:}  If $u^k=u$ $\forall k$ and $\sigma(\mathcal{H})<1$, then the closed-form expression for $x^k$ can be further simplified to give the following tight convergence rate result. 
\begin{prop}
\label{prop:ltiu}
Suppose $\sigma(\mathcal{H})<1$, and $x^k$ is determined by \eqref{eq:lin1}.  If $u^k=u$ $\forall k$, then $x^k$ converges to a limit point $x^\infty=\lim_{k\rightarrow \infty} x^k=(I-\mathcal{H})^{-1}\mathcal{G}u$. And 
  we can compute $x^k$ as 
\begin{align}
\label{eq:ucons}
x^k=x^\infty+(\mathcal{H})^k (x^0-x^\infty).
\end{align}
In addition, $\norm{x^k-x^\infty}\le C_0 (\sigma(\mathcal{H})+\epsilon)^k$ for some $C_0$ and any arbitrarily small $\epsilon>0$.
\end{prop}
From the above proposition, we can clearly see that now  $x^k$ is a sum of a constant steady state term $x^\infty$ and a matrix power term that decays at a linear rate specified by $\sigma(\mathcal{H})$ (see Section 2.2 in \cite{Lessard2014} for more explanations). The convergence rate characterized by $(\sigma(\mathcal{H})+\epsilon)$ is tight. More discussions on the tightness of this convergence rate are provided in the supplementary material.
\item \textbf{Response for exponentially shrinking input:} When $u^k$ itself converges at a linear rate $\tilde{\rho}$ and $\mathcal{H}$ is Schur stable,  $x^k$ will converge to its limit point at a linear rate specified by $\max\{\sigma(\mathcal{H})+\epsilon,\tilde{\rho}\}$. A formal statement is  provided as follows.
\begin{prop}
\label{prop:lti}
Suppose $\sigma(\mathcal{H})<1$, and $x^k$ is determined by \eqref{eq:lin1}. If $u^k$ converges to $u^\infty$ as $\norm{u^k-u^\infty}\le C \tilde{\rho}^k$, then
we  have $x^\infty=\lim_{k\rightarrow \infty} x^k=(I-\mathcal{H})^{-1}\mathcal{G}u^\infty$ and
$\norm{x^k-x^\infty}\le C_0 \left(\max\{\sigma(\mathcal{H})+\epsilon, \tilde{\rho}\}\right)^k$ for some $C_0$ and any arbitrarily small $\epsilon>0$.
\end{prop}
\end{itemize}

The results in Propositions \ref{prop:ltilim}, \ref{prop:ltiu}, and \ref{prop:lti} are well known in the control community. For completeness, we will include their proofs in the supplementary material.

\subsection{Markov jump linear systems}
Another important class of dynamic systems that have been extensively studied in the controls literature is the so-called Markov jump linear system (MJLS) \cite{costa2006}. 
Let $\{z^k\}$ be a Markov chain sampled from a finite state space $\mathcal{S}$.
A MJLS is governed by the following state-space model:
\begin{align}
\label{eq:jump1}
\xi^{k+1}=H(z^k) \xi^k+G(z^k) y^k,
\end{align}
where $H(z^k)$ and $G(z^k)$ are matrix functions of $z^k$.  Here, $\xi^k$ is the state, and $y^k$ is the input. 
There is a one-to-one mapping from $\mathcal{S}$ to the set $\mathcal{N}:=\{1,2,\ldots,n\}$ where $n=|\mathcal{S}|$. We can assume $H(z^k)$ is sampled from a set of matrices $\{H_1, H_2, \ldots, H_n\}$ and $G(z^k)$ is sampled from $\{G_1, G_2, \ldots, G_n\}$. We have $H(z^k)=H_i$ and $G(z^k)=G_i$ when $z^k=i$.
The MJLS theory has been well developed in the controls community \cite{costa2006}. We will apply the MJLS theory to analyze TD learning algorithms.



\section{A general Markov jump system perspective for TD learning}
\label{sec:JumpFor}


In this section, we provide a general jump system perspective for TD learning with linear function approximations.
Notice that many TD learning algorithms including TD, TDC, GTD, GTD2, A-TD, and D-TD  can be modeled by the following linear stochastic recursion:
\begin{align}
\label{eq:m0}
\xi^{k+1}=\xi^k+\alpha \left( A(z^k)\xi^k+b(z^k)\right),
\end{align}
where $\{z^k\}$ forms a finite state Markov chain and  $b(z^k)$ satisfies $\lim_{k\rightarrow \infty}\mathbb{E} b(z^k)=0$.\footnote{This standard assumption is typically related to the projected Bellman equation and can always be enforced by a shifting argument. More explanations are provided in Remark \ref{rem:A1}.} We have $A(z^k)=A_i$ and $b(z^k)=b_i$ when $z^k=i$.
For simplicity, we mainly focus on analyzing \eqref{eq:m0}. Other models including two time-scale schemes \cite{gupta2019finite,xu2019two} will be discussed in the supplementary material.

Our key observation  is that \eqref{eq:m0} can be rewritten as the following MJLS
\begin{align}
\label{eq:m1}
\xi^{k+1}=(I+\alpha A(z^k))\xi^k+\alpha b(z^k).
\end{align}
The above model is a special case of \eqref{eq:jump1} if we set $H(z^k)=I+\alpha A(z^k)$, $G(z^k)=\alpha b(z^k)$, and $y^k=1$ $\forall k$.
Consequently, many TD learning algorithms can be analyzed using the MJLS theory. 

We briefly highlight the analysis idea here.
 Our analysis is built upon the fact that some augmented versions of the mean and the covariance matrix of $\{\xi^k\}$ for the MJLS model \eqref{eq:jump1} actually follow the dynamics of a deterministic LTI model in the form of~\eqref{eq:lin1} \cite[Chapter 3]{costa2006}.  To see this, we denote the transition probabilities for the Markov chain $\{z^k\}$  as $p_{ij}:=\mathbb{P}(z^{k+1}=j\vert z^k=i)$ and specify the transition matrix $P$ by setting its $(i,j)$-th entry to be $p_{ij}$. Obviously, we have $p_{ij}\ge 0$ and $\sum_{j=1}^n p_{ij}=1$ for all $i$.
Next, the indicator function $\textbf{1}_{\{z^k=i\}}$ is defined as  $\textbf{1}_{\{z^k=i\}}=1$ if $z^k=i$ and $\textbf{1}_{\{z^k=i\}}=0$ otherwise. Now we define the following key quantities:
\begin{align*}
q_i^k=\mathbb{E}\left(\xi^k \textbf{1}_{\{z^k=i\}}\right),\,\,\,\,Q_i^k=\mathbb{E}\left( \xi^k (\xi^k)^\tp\textbf{1}_{\{z^k=i\}}\right).
\end{align*} 
Suppose  $y^k=1$ $\forall k$.
Based on \cite[Proposition 3.35]{costa2006}, $q^k$ and $Q^k$  can be iteratively calculated as
\begin{align}
\label{eq:q1S}
q_j^{k+1}&=\sum_{i=1}^n p_{ij} (H_i q_i^k+ G_i p_i^k),\\
\label{eq:co1S}
Q_j^{k+1}&=\sum_{i=1}^n p_{ij} \left(H_i Q_i^k H_i^\tp+2\sym(H_i q_i^k G_i^\tp)+ p_i^k G_i G_i^\tp\right),
\end{align}
where $p_i^k:=\mathbb{P}(z^k=i)$.
If we further augment $q_i^k$ and $Q_i^k$ as
\begin{align*}
q^k=\bmat{q_1^k \\ \vdots \\ q_n^k}, \,\,\,\,Q^k=\bmat{Q_1^k & Q_2^k & \ldots & Q_n^k},
\end{align*}
then it is straightforward to rewrite \eqref{eq:q1S}  \eqref{eq:co1S} as the following LTI system
\begin{align}
\label{eq:jumpGeneral}
\bmat{q^{k+1} \\ \vect(Q^{k+1})}=\bmat{ \mathcal{H}_{11}& 0 \\ \mathcal{H}_{21} & \mathcal{H}_{22}}\bmat{q^k \\ \vect(Q^k)}+\bmat{u_q^k \\ u_Q^k},
\end{align} 
where $\mathcal{H}_{11}$, $\mathcal{H}_{21}$, $\mathcal{H}_{22}$, $u_q^k$, and $u_Q^k$ are given by
\begin{align}
\begin{split}
\mathcal{H}_{11}&= \bmat{p_{11} H_1 & \ldots & p_{n1} H_n\\ \vdots & \ddots & \vdots\\ p_{1n} H_1 & \ldots & p_{nn} H_n},\mathcal{H}_{22}= \bmat{p_{11} H_1\otimes H_1 & \ldots & p_{n1} H_n\otimes H_n\\ \vdots & \ddots & \vdots\\ p_{1n} H_1\otimes H_1 & \ldots & p_{nn} H_n\otimes H_n},\\
\mathcal{H}_{21}&=\bmat{p_{11} (H_1\otimes G_1+G_1\otimes H_1)  & \ldots &  p_{n1} (H_n\otimes G_n+G_n\otimes H_n),  \\ \vdots & \ddots & \vdots \\  p_{1n} (H_1\otimes G_1+G_1\otimes H_1) &  \ldots & p_{nn} (H_n\otimes G_n+G_n\otimes H_n) },\\
u_q^k&= \bmat{p_{11} G_1  & \ldots &  p_{n1} G_n  \\ \vdots & \ddots & \vdots \\  p_{1n} G_1 &  \ldots & p_{nn} G_n }\bmat{p_1^k I_{n_\xi}\\ \vdots \\ p_n^k I_{n_\xi}}, u_Q^k= \bmat{p_{11} G_1\otimes G_1  & \ldots &  p_{n1} G_n\otimes G_n  \\ \vdots & \ddots & \vdots \\  p_{1n} G_1\otimes G_1 &  \ldots & p_{nn} G_n\otimes G_n }\bmat{p_1^k I_{n_\xi^2}\\ \vdots \\ p_n^k I_{n_\xi^2}}.
\end{split}
 \end{align}
A detailed derivation for the above result is presented in the supplementary material.
A key implication here is that $q^k$ and $\vect(Q^k)$ follow the LTI dynamics \eqref{eq:jumpGeneral} and can be analyzed using the standard LTI theory reviewed in Section \ref{sec:LTI}.
Obviously, we have $\mathbb{E} \xi^k=\sum_{i=1}^n q_i^k$, $\mathbb{E} \left(\xi^k (\xi^k)^\tp \right)=\sum_{i=1}^n Q_i^k$, and $\mathbb{E} \norm{\xi^k}^2=\trace(\sum_{i=1}^n Q_i^k)=(\mathbf{1}_n^\tp \otimes \vect(I_{n_\xi})^\tp)\vect(Q^k)$. Hence the mean, covariance, and mean square norm of $\xi^k$ can all be calculated using closed-form expressions.
We will present a detailed analysis for \eqref{eq:m1} and provide related implications for TD learning in the next two sections.

For illustrative purposes, we explain the jump system perspective for the standard TD method.  

\paragraph{Example 1: TD method.} The standard TD method (or TD(0)) uses the following update rule:
\begin{align}
\label{eq:TD0}
\theta^{k+1}&=\theta^k-\alpha \phi(s^k)\left((\phi(s^k)-\gamma\phi(s^{k+1}))^\tp \theta^k-r(s^k)\right),
\end{align}
where $\{s^k\}$ is the underlying Markov chain, $\phi$ is the feature vector, $r$ is the reward, $\gamma$ is the discounting factor, and $\theta^k$ is the weight vector to be estimated.
Suppose $\theta^*$ is the vector that solves the projected Bellman equation.
We can set $z^k=\bmat{(s^{k+1})^\tp & (s^k)^\tp}^\tp$ and then rewrite the TD update as
\begin{align}
\label{eq:TD}
\theta^{k+1}-\theta^*=\left(I+\alpha A(z^k) \right)(\theta^k-\theta^*)+\alpha b(z^k),
\end{align}
where $A(z^k)=\phi(s^k)(\gamma\phi(s^{k+1})-\phi(s^k))^\tp$ and $b(z^k)= \phi(s^k) \left(r(s^k)+(\phi(s^k)-\gamma\phi(s^{k+1}))^\tp \theta^*\right)$. Suppose $\lim_{k\rightarrow\infty} p_i^k=p_i^\infty$.
Since the projected Bellman equation and the equation $\sum_{i=1}^n p_i^\infty b_i=0$ are actually equivalent, we have naturally enforced  $\lim_{k\rightarrow \infty}\mathbb{E} b(z^k)=0$.
Therefore, the TD update can be modeled as \eqref{eq:m1} with $b(z^k)$ satisfying $\lim_{k\rightarrow \infty}\mathbb{E} b(z^k)=0$.
See Section 3.1 in \cite{srikant2019finite} for a similar formulation.
Now we can apply the MJLS theory and the LTI model \eqref{eq:jumpGeneral} to analyze the covariance $\mathbb{E} \left((\theta^k-\theta^*)(\theta^k-\theta^*)^\tp\right)$ and the mean square error $\mathbb{E} \norm{\theta^k-\theta^*}^2$. In this case, we have 
\begin{align*}
q^k=\bmat{\mathbb{E}\left((\theta^k-\theta^*) \textbf{1}_{\{z^k=1\}}\right) \\ \vdots \\ \mathbb{E}\left((\theta^k-\theta^*) \textbf{1}_{\{z^k=n\}}\right) },\,\,\,\,\vect(Q^k)=\bmat{\vect\left(\mathbb{E}( (\theta^k-\theta^*) (\theta^k-\theta^*)^\tp\textbf{1}_{\{z^k=1\}})\right)\\ \vdots \\ \vect\left(\mathbb{E}( (\theta^k-\theta^*) (\theta^k-\theta^*)^\tp\textbf{1}_{\{z^k=n\}})\right)}.
\end{align*} 
Then we can easily analyze $q^k$ and $Q^k$ by applying the LTI model \eqref{eq:jumpGeneral}. 
In general, the covariance matrix $\mathbb{E}\left( (\theta^k-\theta^*) (\theta^k-\theta^*)^\tp\right)$ and the mean value $\mathbb{E}(\theta^k-\theta^*)$ do not directly follow an LTI system. However, when working with the augmented covariance matrix $Q^k$ and the augmented mean value vector $q^k$, we do obtain an LTI model in the form of \eqref{eq:lin1}. Once the closed-form expression for $Q^k$ is obtained, the mean square estimation error for the TD update can be immediately calculated as $\mathbb{E}\norm{\theta-\theta^*}^2=\trace(\sum_{i=1}^n Q_i^k)=(\mathbf{1}_n^\tp \otimes \vect(I_{n_\theta})^\tp)\vect(Q^k)$.

Here we omit the detailed formulations for other TD learning methods. The key message is that $\{z^k\}$ can be viewed as a jump parameter and TD learning methods are essentially just MJLS.
Notice that all the TD learning algorithms that can be analyzed using the ODE method are in the form of~\eqref{eq:m1}. Jump system perspectives for other TD learning algorithms are discussed in the supplementary material.

\begin{rem}[Assumptions]
\label{rem:A1}
Denote $\bar{A}=\lim_{k\rightarrow \infty}\mathbb{E}A(z^k)=\sum_{i=1}^n p_i^\infty A_i$.
In this paper, we will assume $\bar{A}$  is Hurwitz. This assumption is standard and even required by the basic ODE approach. For the standard TD method, $\bar{A}$ is Hurwitz when the discount factor is smaller than $1$, $p_i^\infty$ is positive for all $i$,  and the feature matrix is full column rank. It is worth emphasizing that the assumption $\lim_{k\rightarrow \infty}\mathbb{E}b(z^k)=0$ is also general.
Suppose $\sum_{i=1}^n p_i^\infty b_i\neq 0$. This case can still be handled using a shifting argument since $\bar{A}$ is Hurwitz. Notice
the iteration $\xi^{k+1}=(I+\alpha A(z^k))\xi^k+\alpha b(z^k)$ can be rewritten as $\xi^{k+1}-\tilde{\xi}=\xi^k-\tilde{\xi}+\alpha\left( A(z^k) (\xi^k-\tilde{\xi})+A(z^k) \tilde{\xi}+b(z^k)\right)$ for any $\tilde{\xi}$.  Now we denote $\tilde{b}_i=A_i \tilde{\xi}+b_i$ and the above iteration just becomes $\xi^{k+1}-\tilde{\xi}=(I+\alpha A(z^k))(\xi^k-\tilde{\xi})+\alpha \tilde{b}(z^k)$.
When $\bar{A}$ is Hurwitz (and hence invertible), we can choose $\tilde{\xi}=-(\sum_{i=1}^n p_i^\infty A_i)^{-1} (\sum_{i=1}^n p_i^\infty b_i)$ such that 
 $\sum_{i=1}^n p_i^\infty \tilde{b}_i=\sum_{i=1}^n p_i^\infty (A_i \tilde{\xi}+b_i)=0$. 
\end{rem}

\begin{rem}[Generality of \eqref{eq:jump1}]  Notice that \eqref{eq:jump1} provides a general jump system model for linear stochastic schemes that can be in more complicated forms than \eqref{eq:m0}. However, \eqref{eq:jump1} can not be directly used to cover nonlinear stochastic approximation schemes. See \cite{wang2019multistep, zaiwei2019} for recent finite sample results on nonlinear stochastic approximation over non-IID data.
\end{rem}

%
%


%

\section{Analysis under the IID assumption}

For illustrative purposes, we first present the analysis for \eqref{eq:m1} under the IID assumption ($\mathbb{P}(z^k=i)=p_i$ $\forall i$).  In this case, the analysis is significantly simpler, since $\{\mathbb{E}\xi^k\}$ and $\{\mathbb{E} \left(\xi^k (\xi^k)^\tp \right)\}$ directly form LTI systems with much smaller dimensions. 
We denote $\mu^k:=\mathbb{E} \xi^k$ and $\mathbb{Q}^k:=\mathbb{E} \left(\xi^k (\xi^k)^\tp \right)$.
Then the following equations hold for the general jump system model \eqref{eq:jump1}
\begin{align}
\begin{split}
\label{eq:IID1}
\mu^{k+1}&=\sum_{i=1}^n p_i (H_i \mu^k+G_i)=\bar{H} \mu^k+\bar{G},\\
\vect(\mathbb{Q}^{k+1})&=(\sum_{i=1}^n p_i H_i\otimes H_i) \vect(\mathbb{Q}^k)+\left(\sum_{i=1}^n p_i (H_i\otimes G_i+G_i\otimes H_i) \right) \mu^k+\sum_{i=1}^n p_i G_i\otimes G_i.
\end{split}
\end{align}
There are many ways to derive the above formulas. One way is to first show $q_i^k=p_i \mu^k$ and $Q_i^k=p_i \mathbb{Q}^k$ in this case and then apply \eqref{eq:q1S} and \eqref{eq:co1S}.  Another way is to directly modify the proof of Theorem \ref{thm:IID}.
Now consider the jump system model \eqref{eq:m1} under the assumption $\mathbb{E} b(z^k)=\sum_{i=1}^n p_i b_i=0$.
In this case, we have $H_i=I+\alpha A_i$, $G_i=\alpha b_i$, and $y^k=1$. 
Denote $\bar{A}:=\sum_{i=1}^n p_i A_i$. We can directly obtain the following result.
\begin{thm}
\label{thm:IID}
Consider the jump system model \eqref{eq:m1} with $H_i=I+\alpha A_i$, $G_i=\alpha b_i$, and $y^k=1$.  Suppose $\{z^k\}$ is sampled from $\mathcal{N}$ using an IID distribution $\mathbb{P}(z^k=i)=p_i$. In addition, assume $\sum_{i=1}^n p_i b_i=0$. Then $\mu^k$ and $\vect(\mathbb{Q}^k)$ are governed by the following LTI system:
\begin{align}
\label{eq:LTIIID}
\bmat{\mu^{k+1} \\ \vect(\mathbb{Q}^{k+1})}=\bmat{ \mathcal{H}_{11}& 0 \\ \mathcal{H}_{21} & \mathcal{H}_{22}}\bmat{\mu^k \\ \vect(\mathbb{Q}^k)}+\bmat{0\\ \alpha^2 \sum_{i=1}^n p_i(b_i\otimes b_i)},
\end{align} 
where $\mathcal{H}_{11}$, $\mathcal{H}_{21}$ and $\mathcal{H}_{22}$ are determined as
\begin{align}
\label{eq:HIID}
\begin{split}
\mathcal{H}_{11}&=I+\alpha \bar{A},\\
\mathcal{H}_{21}&=\alpha^2 \sum_{i=1}^n p_i (A_i\otimes b_i+b_i\otimes A_i),\\
\mathcal{H}_{22}&=I_{n_\xi^2}+\alpha (I\otimes \bar{A}+\bar{A}\otimes I)+\alpha^2 \sum_{i=1}^n p_i( A_i\otimes A_i).
 \end{split}
 \end{align}
In addition, if $\sigma(\mathcal{H}_{22})<1$,  we have
\begin{align}
\label{eq:resIIDF}
\begin{split}
\bmat{\mu^k \\ \vect(\mathbb{Q}^k)}&=\left(\bmat{\mathcal{H}_{11} & 0 \\ \mathcal{H}_{21} & \mathcal{H}_{22}}\right)^k\left(\bmat{\mu^0 \\ \vect(\mathbb{Q}^0)}-\bmat{\mu^\infty\\ \vect(\mathbb{Q}^\infty)}\right)+\bmat{\mu^\infty\\ \vect(\mathbb{Q}^\infty)}
\end{split}
\end{align}
where $\mu^\infty=\lim_{k\rightarrow \infty} \mu^k=0$, and $\vect(\mathbb{Q}^\infty)$ is given as
\begin{align}
\label{eq:resIIDQ}
\vect(\mathbb{Q}^\infty)=\lim_{k\rightarrow 0} \vect(\mathbb{Q}^k)= -\alpha\left(I\otimes \bar{A}+\bar{A}\otimes I+\alpha \sum_{i=1}^n p_i( A_i\otimes A_i)\right)^{-1}\left(\sum_{i=1}^n  p_i(b_i\otimes b_i)\right)
\end{align}
\end{thm}
\begin{proof}
 For completeness, a detailed proof is  presented in the supplementary material.
\end{proof}
Now we discuss the implications of the above theorem for TD learning. For simplicity, we denote $\mathcal{H}=\bmat{\mathcal{H}_{11} & 0 \\ \mathcal{H}_{21} & \mathcal{H}_{22}}$.

\paragraph{Stability condition for TD learning.} From the LTI theory, the system \eqref{eq:LTIIID} is stable if and only if $\mathcal{H}$ is Schur stable. We can apply Proposition 3.6 in \cite{costa2006} to show that $\mathcal{H}$ is Schur stable if and only if $\mathcal{H}_{22}$ is Schur stable. Hence, a necessary and sufficient stability condition for the LTI system \eqref{eq:LTIIID} is that $\mathcal{H}_{22}$ is Schur stable.
Under this condition,   the first term on the right side of \eqref{eq:resIIDF} converges to $0$ at a linear rate specified by $\sigma(\mathcal{H})$, and the second term on the right side of \eqref{eq:resIIDF} is a constant matrix quantifying the steady state covariance.
An important question for TD learning is how to choose $\alpha$ such that $\sigma(\mathcal{H}_{22})<1$ for some given $\{A_i\}$, $\{b_i\}$, and $\{p_i\}$. We provide some clue to this question by applying an eigenvalue perturbation analysis to the matrix $\mathcal{H}_{22}$. We assume $\alpha$ is small. Then under mild technical condition\footnote{One such condition is that $\lambda_{\max}(I_{n_\xi^2}+\alpha(I\otimes \bar{A}+\bar{A}\otimes I))$ is a semisimple eigenvalue.}, we can ignore the quadratic term $\alpha^2 \sum_{i=1}^n p_i( A_i\otimes A_i)$ in the expression of  $\mathcal{H}_{22}$ and use $\lambda_{\max}(I_{n_\xi^2}+\alpha (I\otimes \bar{A}+\bar{A}\otimes I))$ to estimate $\lambda_{\max}(\mathcal{H}_{22})$. We have
\begin{align}
\lambda_{\max}(\mathcal{H}_{22})= 1+2\lambda_{\max\real}(\bar{A})\alpha+O(\alpha^2).
\end{align}
Then we immediately obtain $\sigma(\mathcal{H}_{22})\approx 1+2\real(\lambda_{\max\real}(\bar{A}))\alpha+O(\alpha^2)$. Therefore, as long as $\bar{A}$ is Hurwitz, there exists sufficiently small $\alpha$ such that $\sigma(\mathcal{H}_{22})<1$.  
More details of the perturbation analysis are provided in the supplementary material.

\paragraph{Exact limit for the mean square error of TD learning.} Obviously, $\mu^k$ converges to $0$ at the rate specified by $\sigma(I+\alpha \bar{A})$ due to the relation $\mu^k=(I+\alpha \bar{A})^k \mu^0$.
Applying Proposition \ref{prop:lti} and making use of the block structure in $\mathcal{H}$, one can show $\vect(\mathbb{Q}^\infty)=\alpha^2 (I_{n_\xi^2}-\mathcal{H}_{22})^{-1}\left(\sum_{i=1}^n p_i(b_i\otimes b_i)\right)$, which leads to the result in \eqref{eq:resIIDQ}. A key message here is that the covariance matrix converges linearly to an exact limit under the stability condition $\sigma(\mathcal{H}_{22})<1$.
We can clearly see $\lim_{k\rightarrow 0} \vect(\mathbb{Q}^k)=O(\alpha)$ and can be controlled by decreasing $\alpha$.
When $\alpha$ is large, we need to keep the quadratic term $\alpha \sum_{i=1}^n p_i( A_i\otimes A_i)$. Therefore, our theory captures the steady-state behavior of TD learning with both small and large $\alpha$, and complement the existing finite sample bounds in literatures.  To further compare our results with existing finite sample bounds, we obtain the following result for the mean square error of TD learning.

\begin{cor}
Consider the TD update \eqref{eq:TD} with $\bar{A}$ being Hurwitz. Suppose $\sigma(\mathcal{H}_{22})<1$ and $\mathbb{P}(z^k=i)=p_i$ $\forall i$. Then $\lim_{k\rightarrow \infty} \mathbb{E}\norm{\theta^k-\theta^*}^2$ exists and is determined as 
$\delta^\infty:=\lim_{k\rightarrow \infty} \mathbb{E}\norm{\theta^k-\theta^*}^2=\trace(\mathbb{Q}^\infty)$
where $\mathbb{Q}^\infty$ is given by \eqref{eq:resIIDQ}. In addition, the following mean square TD error bounds hold for some constant $C_0$ and any arbitrary small positive $\epsilon$:
\begin{align}
\label{eq:resIIDbound1}
\delta^\infty-C_0 (\sigma(\mathcal{H})+\epsilon)^k\le \mathbb{E}\norm{\theta^k-\theta^*}^2\le \delta^\infty+C_0 (\sigma(\mathcal{H})+\epsilon)^k.
\end{align}
Finally, for sufficiently small $\alpha$, one has $\lim_{k\rightarrow \infty} \mathbb{E}\norm{\theta^k-\theta^*}^2=O(\alpha)$. If $\lambda_{\max}(I_{n_\xi^2}+\alpha(I\otimes \bar{A}+\bar{A}\otimes I))$ is a semisimple eigenvalue, then $\sigma(\mathcal{H})=\sigma(\mathcal{H}_{11})= 1+\real(\lambda_{\max\real}(\bar{A}))\alpha$.  
\end{cor}
\begin{proof}
Recall that we have $\mathbb{E}\norm{\theta^k-\theta^*}^2=\trace(\mathbb{Q}^k)$. Taking limits on both sides leads to the expression for $\delta^\infty$. Then we can apply Proposition \ref{prop:ltiu} to obtain a linear convergence bound for $\mathbb{Q}^k$ which eventually leads to  \eqref{eq:resIIDbound1}.  Finally, we can apply standard matrix perturbation theory to show  $\delta^\infty=O(\alpha)$ and $\sigma(\mathcal{H})=\sigma(\mathcal{H}_{11})= 1+\real(\lambda_{\max\real}(\bar{A}))\alpha$.
\end{proof}
The above corollary gives both upper and lower bounds for the mean square error of TD learning. From the above result,  the final TD estimation error is actually exactly on the order of $O(\alpha)$. This justifies the tightness of the existing upper bounds for the final TD error up to a constant factor. From the above corollary, we can also see that one can increase the convergence rate at the price of increasing the steady state error.
This is consistent with the finite sample bound in the literature \cite{bhandari2018finite,srikant2019finite}. 
 It is also worth mentioning that $(\sigma(\mathcal{H})+\epsilon)$ is a tight estimation for the convergence rate of the   matrix power $(\mathcal{H})^k$. See the supplementary material for more explanations.


\section{Analysis under the  Markov assumption}
\label{sec:Markov}

Now we can analyze the behaviors of TD learning 
 under the general assumption that $\{z^k\}$ is a Markov chain. 
Recall that the augmented mean vector $q^k$ and the augmented covariance matrix $Q^k$ have been defined in Section \ref{sec:JumpFor}.
 We can directly modify \eqref{eq:jumpGeneral} to obtain the following result.
\begin{thm}
\label{thm:Markov}
Consider the jump system model \eqref{eq:m1} with $H_i=I+\alpha A_i$, $G_i=\alpha b_i$, and $y^k=1$.  Suppose $\{z^k\}$ is a Markov chain sampled from $\mathcal{N}$ using the transition matrix $P$. In addition, define $p_i^k=\mathbb{P}(z^k=i)$ and set the augmented vector $p^k:=\bmat{p_1^k & p_2^k & \ldots & p_n^k}^\tp$. Clearly $p^k= (P^\tp)^k p^0$.
Further denote the augmented vectors as $b:=\bmat{b_1^\tp & b_2^\tp & \ldots & b_n^\tp}^\tp$, $\hat{B}=\bmat{(b_1\otimes b_1)^\tp & \ldots & (b_n\otimes b_n)^\tp}^\tp$, and set $S(b_i, A_i):= (b_i\otimes (I+\alpha A_i)+(I+\alpha A_i)\otimes b_i)$.
 Then $q^k$ and $\vect(Q^k)$ are governed by the following  state-space model:
\begin{align}
\label{eq:MarkovSys}
\bmat{q^{k+1} \\ \vect(Q^{k+1})}=\bmat{ \mathcal{H}_{11}& 0 \\ \mathcal{H}_{21} & \mathcal{H}_{22}}\bmat{q^k \\ \vect(Q^k)}+\bmat{\alpha((P^\tp\diag(p_i^k))\otimes I_{n_\xi})b\\ \alpha^2((P^\tp\diag(p_i^k))\otimes I_{n_\xi^2})\hat{B}},
\end{align} 
where $\mathcal{H}_{11}$, $\mathcal{H}_{21}$ and $\mathcal{H}_{22}$ are given by
\begin{align}
\label{eq:HMarkov}
\begin{split}
\mathcal{H}_{11}&=(P^\tp \otimes I_{n_\xi}) \diag(I_{n_\xi}+\alpha A_i),\\
\mathcal{H}_{21}&=\alpha\bmat{p_{11} S(b_1, A_1) & \ldots &  p_{n1} S(b_n, A_n) \\ \vdots & \ddots & \vdots \\  p_{1n} S(b_1, A_1) &  \ldots & p_{nn} S(b_n, A_n)}, \\
\mathcal{H}_{22}&= (P^\tp \otimes I_{n_\xi^2}) \diag((I_{n_\xi}+\alpha A_i)\otimes (I_{n_\xi}+\alpha A_i)).
 \end{split}
 \end{align}
In addition,  the following closed-form solution holds for any $k$
\begin{align}
\label{eq:resMarkov}
\begin{split}
q^k&=(\mathcal{H}_{11})^k q^0+\alpha\sum_{t=0}^{k-1} (\mathcal{H}_{11})^{k-1-t} ((P^\tp\diag(p_i^t))\otimes I_{n_\xi})b,\\
\vect(Q^k)&=(\mathcal{H}_{22})^k \vect(Q^0)+\sum_{t=0}^{k-1}(\mathcal{H}_{22})^{k-1-t} \left(\mathcal{H}_{21}q^t+ \alpha^2((P^\tp\diag(p_i^t))\otimes I_{n_\xi^2})\hat{B}\right),
\end{split}
\end{align}
where $\mathcal{H}_{11}$, $\mathcal{H}_{21}$ and $\mathcal{H}_{22}$ are determined by \eqref{eq:HMarkov}.
\end{thm}
\begin{proof}
A detailed proof is presented in the supplementary material. We present a proof sketch here. Notice \eqref{eq:MarkovSys} is a direct consequence of \eqref{eq:q1S} and
\eqref{eq:co1S} (which are special cases of Proposition 3.35 in \cite{costa2006}).
Specifically, one can verify the following equations using the Markov assumption
\begin{align}
q_j^{k+1}&=\sum_{i=1}^n p_{ij}\left( (I+\alpha A_i)q_i^k+\alpha p_i^k b_i \right),\\
Q_j^{k+1}&=\sum_{i=1}^n p_{ij} \left((I+\alpha A_i )Q_i^k (I+\alpha A_i)^\tp+2\alpha \sym((I+\alpha A_i) q_i^k b_i^\tp)+\alpha^2 p_i^k b_i b_i^\tp  \right).
\end{align}

Then we can apply the basic property of the vectorization operation $\vect$ to obtain \eqref{eq:MarkovSys}. 
 Applying \eqref{eq:res1} to iterate \eqref{eq:MarkovSys} directly leads to \eqref{eq:resMarkov}.
\end{proof}

Therefore, the evolutions of $q^k$ and $Q^k$ can be fully understood via the well-established LTI system theory.
Now we discuss the implications of Theorem \ref{thm:Markov} for TD learning.

\paragraph{Stability condition for TD learning.} Similar to the IID case, 
the necessary and sufficient stability condition is  $\sigma(\mathcal{H}_{22})<1$. Now $\mathcal{H}_{22}$ becomes a much larger matrix depending on the transition matrix $P$.
An important question is how to choose $\alpha$ such that $\sigma(\mathcal{H}_{22})<1$ for some given $\{A_i\}$, $\{b_i\}$, $P$, and $\{p^0\}$. Again, we perform an eigenvalue perturbation analysis for the matrix $\mathcal{H}_{22}$. This case is quite subtle due to the fact that we are no longer perturbing an identity matrix. We are perturbing the matrix $(P^\tp\otimes I_{n_\xi^2})$ and the eigenvalues here are not simple. Under the ergodicity assumption, the largest eigenvalue for $(P^\tp\otimes I_{n_\xi^2})$ (which is $1$) is semisimple. Hence we can directly apply the results in Section II of  \cite{kato2013perturbation} or Theorem 2.1 in \cite{moro1997lidskii} to show
\begin{align}
\lambda_{\max}(\mathcal{H}_{22})=1+2\lambda_{\max\real}(\bar{A})\alpha+o(\alpha).
\end{align}
where $\bar{A}=\sum_{i=1}^n p_i^{\infty} A_i$ and $p^\infty$ is the unique stationary distribution of the Markov chain under the ergodicity assumption. 
Then we still have $\sigma(\mathcal{H}_{22})\approx 1+2\real(\lambda_{\max\real}(\bar{A}))\alpha+o(\alpha)$. Therefore, as long as $\bar{A}$ is Hurwitz, there exists sufficiently small $\alpha$ such that $\sigma(\mathcal{H}_{22})<1$. This is consistent with Assumption 3  in \cite{srikant2019finite}. To understand the details of our perturbation argument, we refer the readers to the remark placed after Theorem 2.1 in \cite{moro1997lidskii}. Notice we have
\begin{align*} 
\mathcal{H}_{22}=P^\tp\otimes I_{n_\xi^2}+\alpha(P^\tp\otimes I_{n_\xi^2})(A_i\otimes I+I\otimes A_i)+O(\alpha^2).
\end{align*}
The largest eigenvalue of $P^\tp\otimes I_{n_\xi^2}$ is semisimple due to the ergodicity assumption. Then the perturbation result directly follows as a consequence of 
Theorem 2.1 in \cite{moro1997lidskii}. More explanations are also provided in the supplementary material.

\paragraph{Exact limit for the mean square TD error and related convergence rate.}  
Assume the Markov chain is geometrically ergodic, and then $p^t\rightarrow p^{\infty}$ at some linear rate where $p^\infty$ is the stationary distribution. In this case, we can apply Proposition \ref{prop:lti} to show that the mean square error of TD learning converges linearly to an exact limit.
\begin{cor}
Consider the TD update \eqref{eq:TD} with $\bar{A}$ being Hurwitz. Let $\{z^k\}$ be a Markov chain sampled from $\mathcal{N}$ using the transition matrix $P$. Suppose $\sigma(\mathcal{H}_{22})<1$. We set $N=n n_{\xi}^2$. If we assume $p^k\rightarrow p^\infty$ where $p^\infty$ is the stationary distribution for $\{z^k\}$, then we have 
\begin{align}
\label{eq:exactTD}
\begin{split}
q^{\infty}&=\lim_{k\rightarrow \infty} q^k=\alpha(I-\mathcal{H}_{11})^{-1}((P^\tp\diag(p_i^\infty))\otimes I_{n_\xi})b,\\
\vect(Q^\infty)&=\lim_{k\rightarrow 0} \vect(Q^k)=\alpha^2 (I_{N}-\mathcal{H}_{22})^{-1}\left(\alpha^{-2}\mathcal{H}_{21}q^{\infty}+((P^\tp\diag(p_i^\infty))\otimes I_{n_\xi^2})\hat{B}\right),\\
\delta^\infty&=\lim_{k\rightarrow \infty} \mathbb{E}\norm{\theta^k-\theta^*}^2=(\mathbf{1}_n^\tp \otimes \vect(I_{n_\theta})^\tp)\vect(Q^\infty).
\end{split}
\end{align} 
If we further assume the geometric ergodicity, i.e. $\norm{p^k-p^\infty}\le C \tilde{\rho}^k$, then we have
\begin{align}
\label{eq:TDboundMarkov}
\delta^\infty -C_0 \max\{\sigma(\mathcal{H})+\epsilon, \tilde{\rho}\}^k.\le \mathbb{E}\norm{\theta^k-\theta^*}^2\le \delta^\infty +C_0 \max\{\sigma(\mathcal{H})+\epsilon, \tilde{\rho}\}^k.
\end{align}
where $C_0$ is some constant and $\epsilon$ is an arbitrary small positive number.
For sufficiently small $\alpha$, we have $\delta^\infty=O(\alpha)$ and $\sigma(\mathcal{H})=1+\real(\lambda_{\max\real}(\bar{A}))\alpha+o(\alpha)$.
\end{cor}
\begin{proof}
Notice $\mathbb{E} \norm{\theta^k-\theta^*}^2=(\mathbf{1}_n^\tp \otimes \vect(I_{n_\xi})^\tp)\vect(Q^k)$. We can directly apply Theorem \ref{thm:Markov}, Proposition \ref{prop:ltilim}, and Proposition \ref{prop:lti} to prove \eqref{eq:exactTD} and \eqref{eq:TDboundMarkov}. When $\alpha$ is small, we can apply the Laurent series trick in  \cite{avrachenkov2013analytic, gonzalez2015laurent} to show that
$\lim_{k\rightarrow 0} \vect(Q^k)=O(\alpha)$ and $\delta^\infty=O(\alpha)$. 
The difficulty here is that $I_{N}-P^\tp\otimes I_{n_\xi^2}$ is a singular matrix and hence $(I_{N}-\mathcal{H}_{22})^{-1}$ does not have a Taylor series around $\alpha=0$. Therefore, we need to apply some recent matrix inverse perturbation result to perform a Laurent expansion of $(I_{N}-\mathcal{H}_{22})^{-1}$. From the ergodicity assumption, we know the singularity order of $I_{N}-P^\tp\otimes I_{n_\xi^2}$ is just $1$.
Applying Theorem 1 in \cite{avrachenkov2013analytic}, we can obtain the Laurent expansion of $(I_{N}-\mathcal{H}_{22})^{-1}$ and show $(I_{N}-\mathcal{H}_{22})^{-1}=\alpha^{-1} B_{-1}+B_0+\alpha B_1+O(\alpha^2)$. Consequently, we have 
 $\lim_{k\rightarrow 0} \vect(Q^k)=O(\alpha)$ and $\delta^\infty=O(\alpha)$. By applying Theorem 2.1 in \cite{moro1997lidskii}, we can show $\sigma(\mathcal{H})=1+\real(\lambda_{\max\real}(\bar{A}))\alpha+o(\alpha)$.
\end{proof}

Due to the assumption $\sum_{i=1}^n p_i^{\infty} b_i=0$, we have $\lim_{k\rightarrow \infty} q^k\neq 0$ in general but $\mu^\infty=0$. 
Again, we have obtained both upper and lower bounds for the mean square TD error. From the above result,  the final TD error is actually exactly on the order of $O(\alpha)$. This justifies the tightness of the existing upper bounds for the final TD error \cite{bhandari2018finite,srikant2019finite} up to a constant factor. From the above corollary, we can also see the trade-off between the convergence rate and the steady state error.
 Clearly, the convergence rate in \eqref{eq:TDboundMarkov} also depends on the initial distribution $p^0$ and the mixing rate of the underlying Markov jump parameter $\{z^k\}$ (which is denoted as $\tilde{\rho}$). 
If the initial distribution is the stationary distribution, i.e. $p^0=p^\infty$, the input to the LTI dynamical system \eqref{eq:MarkovSys} is just a constant for all $k$ and then we will be able to obtain an exact formula similar to \eqref{eq:resIIDF}. However, for a general initial distribution $p^0$, the mixing rate $\tilde{\rho}$ matters more and may affect the overall convergence rate. One resultant guideline for algorithm design is that increasing $\alpha$  may not increase the convergence rate when the mixing rate $\tilde{\rho}$ dominates the convergence process.
When $\alpha$ becomes smaller and smaller, eventually $\sigma(\mathcal{H})$ is going to become the dominating term and the mixing rate does not affect the convergence rate any more. 

\paragraph{Algorithm design.} Here we make a remark on how our proposed MJLS framework can be further extended to provide clues for designing fast TD learning.
When $\alpha$ (or even other hyperparameters including momentum term) is changing with time, we can still obtain expressions of $\vect(Q^k)$ and $q^k$ in an iterative form. However, both $\mathcal{H}$ and $\mathcal{G}$ depend on $k$ now. 
Then given a fixed time budget $T$, in theory it is possible to minimize the mean square estimation error at $T$ subject to some optimization constraints in the form of a time-varying iteration $\bmat{q^{k+1}\\ \vect(Q^{k+1})}=\mathcal{H}(k) \bmat{q^k\\ \vect(Q^k)}+\mathcal{G}(k)u^k$. One may use this control-oriented optimization formulation to gain some theoretical insights on how to choose hyperparameters adaptively for fast TD learning. Clearly, solving such an optimization problem requires knowing the underlying Markov model. However, this type of theoretical study may lead to new hyperparameter tuning heuristics that do not require the model information.


%
%
\subsection*{Acknowledgments}

Bin Hu acknowledges helpful discussions with Rayadurgam Srikant.

\bibliographystyle{plain}
\small

\bibliography{neurips_2019}
\clearpage
\setcounter{section}{0}
\renewcommand{\thesection}{\Alph{section}}

\begin{center}
{\Large\bf Supplementary Material}
\end{center}
\numberwithin{equation}{section}
\numberwithin{thm}{section}
\numberwithin{lem}{section}
\numberwithin{prop}{section}

\section{More facts about LTI systems}

\subsection{Tightness of the spectral radius stability condition}
\label{sec:A1}
 Technically speaking, the condition $\sigma(\mathcal{H})<1$ is necessary and sufficient for the asymptotic stability of the LTI system \eqref{eq:lin1}. Since exponential stability and asymptotic stability are equivalent notions of stability for LTI systems, the condition $\sigma(\mathcal{H})<1$ is also necessary and sufficient for the exponential stability of \eqref{eq:lin1}. See Theorem 8.3 and Theorem 8.4 in \cite{hespanha2009} for formal statements of these facts. We will give a more intuitive explanation here. Specifically, we look at the behaviors of the matrix power term $(\mathcal{H})^k x^0$.
This is the homogeneous state response of \eqref{eq:lin1}. There are three possible behaviors for this term. 
\begin{enumerate}
\item When $\mathcal{H}$ is Schur stable (or equivalently $\sigma(\mathcal{H})<1$), the term $(\mathcal{H})^k$ converges to a zero matrix and $(\mathcal{H})^k x^0\rightarrow 0$ for any arbitrary $x^0$. The convergence rate is linear and is completely specified by the spectral radius $\sigma(\mathcal{H})$.
\item When $\sigma(\mathcal{H})\le 1$ and all
the Jordan blocks corresponding to eigenvalues with magnitude equal to $1$ are
$1\times 1$, $(\mathcal{H})^k$ remains bounded for any $k$. This is the so-called marginal stability case where $(\mathcal{H})^k x^0$ remains bounded but may not converge to $0$.
\item For all other cases, $(\mathcal{H})^k$ is unbounded and there exists $x^0$ such that $(\mathcal{H})^k x^0\rightarrow \infty$.
\end{enumerate}
See Section 7.2 in \cite{hespanha2009} for a detailed explanation of the above fact.
Consequently, we can only guarantee $(\mathcal{H})^k x^0$ to converge for all $x^0$ when $\sigma(\mathcal{H})<1$.
Therefore, the condition $\sigma(\mathcal{H})<1$ is a tight condition for the stability of the LTI system \eqref{eq:lin1}.
As mentioned above, when $\sigma(\mathcal{H})<1$, $(\mathcal{H})^k x^0$ converges at a linear rate completely determined by $\sigma(\mathcal{H})$. Technically speaking, the convergence rate is either equal to $\sigma(\mathcal{H})+\epsilon$ for some arbitrary small positive $\epsilon$ or just equal to $\sigma(\mathcal{H})$ itself.
Now we provide a detailed discussion on this convergence rate.

\subsection{Convergence rate of the matrix power}
\label{sec:rateH}

Notice we have $(\mathcal{H})^k x^0=\rho^k \left(\rho^{-1}\mathcal{H}\right)^k x^0$.
As long as $\left(\rho^{-1}\mathcal{H}\right)^k x^0$ stays bounded for any $x^0$, the term  $(\mathcal{H})^k x^0$ will converge at the linear rate $\rho$.

From  Section 7.2 in \cite{hespanha2009}, we can almost directly see how to determine the convergence rate of $(\mathcal{H})^k x^0$.
\begin{enumerate}
\item When $\sigma(\mathcal{H})< 1$ and all
the Jordan blocks corresponding to eigenvalues with magnitude equal to $\sigma(\mathcal{H})$ are
$1\times 1$,  we can choose $\rho=\sigma(\mathcal{H})$ and show $\sigma(\rho^{-1}\mathcal{H})\le 1$ and all
the Jordan blocks corresponding to eigenvalues (of  $\rho^{-1}\sigma(\mathcal{H})$) with magnitude equal to $1$ are
$1\times 1$.  Then $(\rho^{-1}\mathcal{H})^k x^0$ remains bounded for all $x^0$ and hence $(\mathcal{H})^k x^0=\rho^k \left(\rho^{-1}\mathcal{H}\right)^k x^0$ converges at a linear rate $\rho=\sigma(\mathcal{H})$.

\item  When $\sigma(\mathcal{H})< 1$ and some of
the Jordan blocks corresponding to eigenvalues with magnitude equal to $\sigma(\mathcal{H})$ are not
$1\times 1$, we need to choose $\rho=\sigma(\mathcal{H})+\epsilon$ for some arbitrary small $\epsilon>0$. Then $\sigma(\rho^{-1}\mathcal{H})< 1$ and $(\rho^{-1}\mathcal{H})^k x^0$ converges to $0$ for all $x^0$. Consequently,  $(\mathcal{H})^k x^0=\rho^k \left(\rho^{-1}\mathcal{H}\right)^k x^0$ converges at a linear rate $\rho=\sigma(\mathcal{H})+\epsilon$.
\end{enumerate}

In this paper, for simplicity we do not want to further look at the Jordan decomposition structure of $\mathcal{H}$ and hence we always set the rate as $\rho=\sigma(\mathcal{H})+\epsilon$.
One can also use the relationship between spectral radius and other matrix norms to obtain the above convergence rate. See Section 2.2 in \cite{Lessard2014} for a detailed discussion. The arbitrarily small number $\epsilon$ also appears in the argument there.

\section{More discussions about Markov jump linear systems}

First, we verify that  \eqref{eq:q1S} and \eqref{eq:co1S} are equivalent to the LTI model \eqref{eq:jumpGeneral}.

Rewriting \eqref{eq:q1S} as an LTI model is quite trivial. Rewriting \eqref{eq:co1S} as an LTI system requires applying the vectorization operation to obtain the following formula,
\begin{align}
\begin{split}
\bmat{\vect(Q_1^{k+1}) \\ \vdots \\ \vect(Q_n^{k+1})}&=\bmat{p_{11} H_1\otimes H_1 & \ldots & p_{n1} H_n\otimes H_n\\ \vdots & \ddots & \vdots\\ p_{1n} H_1\otimes H_1 & \ldots & p_{nn} H_n\otimes H_n}\bmat{\vect(Q_1^k) \\ \vdots \\ \vect(Q_n^k)}+\bmat{p_{11}p_1^k I_{n_\xi} & \ldots & p_{n1} p_n^k I_{n_\xi^2}\\ \vdots & \ddots & \vdots \\p_{1n}p_1^k I_{n_\xi^2} & \ldots & p_{nn} p_n^k I_{n_\xi^2} }\bmat{G_1\otimes G_1 \\ \vdots \\ G_n\otimes G_n}\\
&+\bmat{p_{11} (G_1\otimes H_1+H_1\otimes G_1) & \ldots & p_{n1} (G_n\otimes H_n+H_n\otimes G_n)\\ \vdots & \ddots & \vdots \\p_{1n} (G_1\otimes H_1+H_1\otimes G_1) & \ldots & p_{nn} (G_n\otimes H_n+H_n\otimes G_n) }\bmat{q_1^k \\ \vdots \\ q_n^k}.
\end{split}
\end{align}
Then we can augment the update rules for $q^k$ and $\vect(Q^k)$ to obtain the desired LTI model for $(q^k, \vect(Q^k))$.

Now we briefly review how to analyze $q^k$ and $Q^k$ using the LTI model \eqref{eq:jumpGeneral}.
We denote $\mathcal{H}=\bmat{ \mathcal{H}_{11}& 0 \\ \mathcal{H}_{21} & \mathcal{H}_{22}}$.

First, we will have the following closed-form formula for computing $(q^k, \vect(Q^k))$:
\begin{align}
\bmat{q^k \\ \vect(Q^k)}=\left(\bmat{ \mathcal{H}_{11}& 0 \\ \mathcal{H}_{21} & \mathcal{H}_{22}}\right)^k\bmat{q^0 \\ \vect(Q^0)}+\sum_{t=0}^{k-1} \left(\bmat{ \mathcal{H}_{11}& 0 \\ \mathcal{H}_{21} & \mathcal{H}_{22}}\right)^{k-1-t} \bmat{u_q^t \\ u_Q^t}.
\end{align}
The first term on the right side of the above equation will be guaranteed to converge to $0$ if we have the stability condition $\sigma(\mathcal{H})<1$. As discussed in \ref{sec:A1}, the stability condition $\sigma(\mathcal{H})<1$ is fairly tight. Based on Proposition~3.6 in \cite{costa2006}, we know that $\mathcal{H}$ is Schur stable if and only if $\mathcal{H}_{22}$ is Schur stable. Therefore,  the needed stability condition is   $\sigma(\mathcal{H}_{22})<1$. Under this condition, if we have $p_i^k\rightarrow p_i^\infty$, then Statement 2 in Proposition \ref{prop:lti} can be used to show
\begin{align}
\begin{split}
&u_q^\infty=\lim_{k\rightarrow \infty}u_q^k= \bmat{p_{11} G_1  & \ldots &  p_{n1} G_n  \\ \vdots & \ddots & \vdots \\  p_{1n} G_1 &  \ldots & p_{nn} G_n }\bmat{p_1^\infty I_{n_\xi}\\ \vdots \\ p_n^\infty I_{n_\xi}},\\
&u_Q^\infty=\lim_{k\rightarrow \infty}u_Q^k= \bmat{p_{11} G_1\otimes G_1  & \ldots &  p_{n1} G_n\otimes G_n  \\ \vdots & \ddots & \vdots \\  p_{1n} G_1\otimes G_1 &  \ldots & p_{nn} G_n\otimes G_n }\bmat{p_1^\infty I_{n_\xi^2}\\ \vdots \\ p_n^\infty I_{n_\xi^2}},\\
&\bmat{q^\infty\\ \vect(Q^\infty)}=\lim_{k\rightarrow \infty} \bmat{q^k\\ \vect(Q^k)}=\left( I-\mathcal{H}\right)^{-1} \bmat{u_q^\infty\\ u_Q^\infty}.
\end{split}
 \end{align}

Finally, if $\norm{p^k-p^\infty}\le C_p \tilde{\rho}^k$, then there exists a constant $C$ such that the following inequality holds.
\begin{align*}
\norm{\bmat{u_q^k\\ u_Q^k}-\bmat{u_q^\infty\\ u_Q^\infty}}\le C \tilde{\rho}^k.
\end{align*}
If we know $\sigma(\mathcal{H}_{22})<1$, then we can directly apply
Proposition \ref{prop:lti} to obtain the following linear convergence result:
\begin{align*}
\norm{\bmat{q^k\\ \vect(Q^k)}-\bmat{q^\infty \\ \vect(Q^\infty)}} \le C_0 \max\{\sigma(\mathcal{H})+\epsilon, \tilde{\rho}\}^k,
\end{align*}
where $C_0$ is some constant and $\epsilon$ is an arbitrarily small positive number.
We can see that the convergence rates of $\vect(Q^k)$ and $q^k$ depend on both $\sigma(\mathcal{H})$ and the mixing rate of the underlying Markov jump parameter $\{z^k\}$ (which is denoted as $\tilde{\rho}$).

Therefore, when the underlying Markov chain $\{z^k\}$ is aperiodic and irreducible, the mean and covariance of the MJLS just converges to the steady state values at a linear rate specified by $ \max\{\sigma(\mathcal{H})+\epsilon, \tilde{\rho}\}$. 
This is a powerful result that can be potentially applied to more general stochastic approximation schemes other than \eqref{eq:m0}. We also want to mention that 
there are other proofs for the convergence of $(q^k, Q^k)$. See Proposition 3.36 in \cite{costa2006} for an alternative proof. Here, our result is a little bit stronger than   Proposition 3.36 in \cite{costa2006}  since we also specify the convergence rate of $(q^k, Q^k)$.

Finally, it is worth mentioning that under the assumption $\sigma(\mathcal{H}_{22})<1$, one can further prove $\{Q^k\}$ converges to a stationary process in some sense. This is a stronger result. Specifically,
Proposition 3.37 In \cite{costa2006} shows that the MJLS is ``asymptotically wide sense stationary" under the assumption $\sigma(\mathcal{H}_{22})<1$. We are not that interested in the correlation between the updates at different steps since our main purpose is analyzing TD learning. Hence we will skip a detailed discussion on this topic. See Chapter 3.4 in \cite{costa2006} for a thorough treatment.

\section{Detailed proofs}

\subsection{Detailed proofs of Propositions \ref{prop:ltilim}, \ref{prop:ltiu}, and \ref{prop:lti}}

We believe that all the statements in Propositions \ref{prop:ltilim}, \ref{prop:ltiu}, and \ref{prop:lti} are known in the controls field. Since we are not able to find a reference to exactly match the statements, we provide a proof here for completeness.

We will need the following lemma.

\begin{lem}
\label{lem:help1}
Consider the LTI model \eqref{eq:lin1}.
Suppose $\sigma(\mathcal{H})<1$. We set $\rho=\sigma(\mathcal{H})+\epsilon$ where $\epsilon$ is an arbitrary small positive number. Then there exists a positive definite matrix $V$ and a positive constant $\gamma$ s.t. the following inequality holds for all $k$,
\begin{align}
\label{eq:iter1}
(x^{k+1})^\tp V x^{k+1}\le \rho^2 (x^k)^\tp V x^k+\gamma \norm{u^k}^2.
\end{align}
\end{lem}
\begin{proof}
We know $\rho^{-1}\mathcal{H}$ is Schur stable. Based on Theorem 8.4 in \cite{hespanha2009}, there exists a positive definite matrix $V$ such that
\begin{align*}
\rho^{-2}\mathcal{H}^\tp V \mathcal{H}-V< 0,
\end{align*}
where the matrix inequality holds in the negative definite sense. The above condition is actually equivalent to $\mathcal{H}^\tp V \mathcal{H}-\rho^2 V<0$. Choose the matrix $V$ that satisfies $\mathcal{H}^\tp V \mathcal{H}-\rho^2 V<0$. Then there exists a sufficiently large $\gamma$ such that $\mathcal{G}^\tp V \mathcal{G}-\gamma I<0$ and 
$\mathcal{H}^\tp V \mathcal{H}-\rho^2 \mathcal{H}-\mathcal{H}^\tp V \mathcal{G} (\mathcal{G}^\tp V \mathcal{G}-\gamma I)^{-1} \mathcal{G}^\tp V \mathcal{H}<0$. By Schur complement lemma, this is equivalent to
\begin{align*}
\bmat{\mathcal{H}^\tp V \mathcal{H}-\rho^2 V & \mathcal{H}^\tp V \mathcal{G} \\ \mathcal{G}^\tp V \mathcal{H} & \mathcal{G}^\tp V \mathcal{G}}+\bmat{0 & 0 \\ 0 & -\gamma I}< 0.
\end{align*}
Now we left and right multiply the right side of the above matrix inequality with $\bmat{(x^k)^\tp & (u^k)^\tp}$ and $\bmat{(x^k)^\tp & (u^k)^\tp}^\tp$. This leads to
\begin{align*}
\bmat{x^k \\ u^k}^\tp \bmat{\mathcal{H}^\tp V \mathcal{H}-\rho^2 V & \mathcal{H}^\tp V \mathcal{G} \\ \mathcal{G}^\tp V \mathcal{H} & \mathcal{G}^\tp V \mathcal{G}}\bmat{x^k \\ u^k}+\bmat{x^k\\ u^k}^\tp \bmat{0 & 0 \\ 0 & -\gamma I}\bmat{x^k \\ u^k}\le 0.
\end{align*}
One can verify that the first term on the left side of the above inequality is just equal to $(x^{k+1})^\tp V x^{k+1}- \rho^2 (x^k)^\tp V x^k$ as follows
\begin{align*}
(x^{k+1})^\tp V x^{k+1}- \rho^2 (x^k)^\tp V x^k&= (\mathcal{H} x^k+\mathcal{G} u^k)^\tp  V(\mathcal{H} x^k+\mathcal{G} u^k)-\rho^2 (x^k)^\tp V x^k\\
&=\bmat{x^k \\ u^k}^\tp \bmat{\mathcal{H}^\tp V \mathcal{H}-\rho^2 V & \mathcal{H}^\tp V \mathcal{G} \\ \mathcal{G}^\tp V \mathcal{H} & \mathcal{G}^\tp V \mathcal{G}}\bmat{x^k \\ u^k}.
\end{align*}
Therefore, we have $(x^{k+1})^\tp V x^{k+1}- \rho^2 (x^k)^\tp V x^k-\gamma \norm{u^k}^2\le 0$, which is equivalent to \eqref{eq:iter1}.
\end{proof}

Now we are ready to prove Proposition \ref{prop:ltilim}.
 Since $\sigma(\mathcal{H})<1$,  $x^\infty$ can still be well defined as $x^\infty=(I-\mathcal{H})^{-1}\mathcal{G}u$. Notice we have not shown the existence of $\lim_{k\rightarrow \infty} x^k$ at this point. We will show $\lim_{k\rightarrow \infty} x^k$ exists and is equal to $x^\infty$. Applying the relation $(I-\mathcal{H})x^\infty=\mathcal{G}u^\infty$, we still have
\begin{align*}
x^{k+1}-x^\infty=\mathcal{H}(x^k-x^\infty)+\mathcal{G} (u^k-u^\infty).
\end{align*}
By Lemma \ref{lem:help1}, there exists a positive definite matrix $V$ and a positive number $\gamma$ such that
\begin{align}
\label{eq:iter2}
(x^{k+1}-x^\infty)^\tp V (x^{k+1}-x^\infty)\le \rho^2 (x^k-x^\infty)^\tp V (x^k-x^\infty)+\gamma \norm{u^k-u^\infty}^2,
\end{align}
where $\rho=\sigma(\mathcal{H})+\epsilon<1$.
First we show $(x^k-x^\infty)^\tp V (x^k-x^\infty)$ is bounded for all $k$ and then we apply $\limsup$ to the above inequality.
Clearly there exists a constant $U$ such that $\norm{u^k-u^\infty}^2\le U$ for all $k$. Then we have
\begin{align*}
(x^k-x^\infty)^\tp V (x^k-x^\infty)\le \rho^{2k} (x^0-x^\infty)^\tp V (x^0-x^\infty)+ \sum_{t=0}^\infty \rho^{2t} \gamma U\le  (x^0-x^\infty)^\tp V (x^0-x^\infty)+ \frac{ \gamma U}{1-\rho^2}.
\end{align*}
Therefore, $\limsup_{k\rightarrow \infty}(x^k-x^\infty)^\tp V (x^k-x^\infty)$ is finite.
Now we take $\limsup$ on both sides of \eqref{eq:iter2} and will immediately be able to show $\limsup_{k\rightarrow \infty} (x^k-x^\infty)^\tp V (x^k-x^\infty)=0$. Since $V$ is positive definite, we have $x^k\rightarrow x^\infty$, which is the desired conclusion.

Next, we prove Proposition \ref{prop:ltiu}. 
If $u^k=u$ for all $k$, we have $\sum_{t=0}^{k-1} (\mathcal{H})^{k-1-t} \mathcal{G} u^t=\sum_{t=0}^{k-1} (\mathcal{H})^t \mathcal{G} u$.
When $\sigma(\mathcal{H})<1$, we have $(\mathcal{H})^k\rightarrow 0$ and $\sum_{k=0}^\infty (\mathcal{H})^k=(I-\mathcal{H})^{-1}$. Hence we have $x^\infty=\lim_{k\rightarrow \infty} x^k=(I-\mathcal{H})^{-1}\mathcal{G}u$. Clearly, $(I-\mathcal{H})$ is nonsingular due to the fact $\sigma(\mathcal{H})<1$. Therefore, we have $(I-\mathcal{H}) x^\infty=\mathcal{G}u^\infty=\mathcal{G} u$. Now it is straightforward to show
\begin{align*}
x^{k+1}-x^\infty=\mathcal{H} x^k+\mathcal{G} u^k-x^\infty=\mathcal{H}(x^k-x^\infty)+\mathcal{G} u^k-(I-\mathcal{H})x^\infty=\mathcal{H}(x^k-x^\infty),
\end{align*}
which directly leads to the desired conclusion $x^k=x^\infty+(\mathcal{H})^k (x^0-x^\infty)$.

Finally, we will still use \eqref{eq:iter2} to prove Proposition \ref{prop:lti}. It is assumed that the arbitrary small $\epsilon$ is chosen in a way that $\epsilon+\sigma(\mathcal{H})\neq \tilde{\rho}$ since one can always decrease $\epsilon$ by a tiny bit.
Then  iterating \eqref{eq:iter2} leads to
\begin{align*}
(x^k-x^\infty)^\tp V (x^k-x^\infty)&\le \rho^{2k} (x^0-x^\infty)^\tp V (x^0-x^\infty)+\gamma \sum_{t=0}^{k-1} \rho^{2(k-1-t)} \norm{u^t-u^\infty}^2\\
&\le \rho^{2k} (x^0-x^\infty)^\tp V (x^0-x^\infty)+C^2\gamma\sum_{t=0}^{k-1}\rho^{2(k-1-t)}\tilde{\rho}^{2t}\\
&=\rho^{2k} (x^0-x^\infty)^\tp V (x^0-x^\infty)+\left(\frac{C^2\gamma}{\rho^2-\tilde{\rho}^2}\right)(\rho^{2k}-\tilde{\rho}^{2k}).
\end{align*}
Obviously the right side of the above inequality is on the order of $O\left((\max\{\rho, \tilde{\rho}\})^{2k}\right)$. Hence we have
\begin{align*}
\norm{x^k-x^\infty}^2\le \frac{1}{\lambda_{\min}(V)} (x^k-x^\infty)^\tp V (x^k-x^\infty)=O\left((\max\{\rho, \tilde{\rho}\})^{2k}\right),
\end{align*}
which leads to $\norm{x^k-x^\infty}=O\left((\max\{\rho, \tilde{\rho}\})^{k}\right)$. This completes the proof of this proposition.

An interesting thing is that 
when $\rho=\tilde{\rho}$, the convergence rate is actually on the order of $O(k\rho^k)$. Specifically, 
we have
\begin{align*}
(x^k-x^\infty)^\tp V (x^k-x^\infty) &\le \rho^{2k} (x^0-x^\infty)^\tp V (x^0-x^\infty)+C^2\gamma k\rho^{2(k-1)}.
\end{align*}
Of course this rate is always bounded above by $O((\epsilon+\rho)^k)$. In addition, if it happens $\epsilon+\sigma(\mathcal{H})=\tilde{\rho}$, one can always decrease $\epsilon$ by a tiny bit and the convergence rate becomes linear again.

\subsection{A detailed proof for Theorem \ref{thm:IID}}

The underlying probability space is
denoted by $(\Omega, \mathcal{F}, \mathbb{P})$. We denote by
$\mathcal{F}_k$ the $\sigma$-algebra generated by $(z^0, z^1, \ldots,
z^k)$. Clearly, $z^k$ is $\mathcal{F}_k$-adapted and we obtain a
filtered probability space $(\Omega, \mathcal{F}, \{\mathcal{F}_k\},
\mathbb{P})$ on which the stochastic optimization method is defined.

First, we prove $\mu^{k+1}=(I+\alpha \bar{A})\mu^k$. Since $\mathbb{E} b(z^k)=\sum_{i=1}^n p_i b_i=0$, we have
\begin{align*}
\mathbb{E}(\xi^{k+1}\vert \mathcal{F}_{k-1})=\sum_{i=1}^n p_i\left((I+\alpha A_i) \xi^k+\alpha b_i\right)=\left(I+\alpha(\sum_{i=1}^n p_i A_i)\right)\xi^k+\alpha\sum_{i=1}^n p_i b_i=(I+\alpha \bar{A})\xi^k.
\end{align*}
Taking full expectation of the above equation leads to $\mu^{k+1}=(I+\alpha \bar{A}) \mu^k$.

Next, we prove $\mathbb{Q}^{k+1}=\mathbb{Q}^k+\alpha (\bar{A} \mathbb{Q}^k+ \mathbb{Q}^k \bar{A}^\tp)+\alpha^2 \sum_{i=1}^n p_i (A_i \mathbb{Q}^k A_i^\tp +2\sym(A_i \mu^k  b_i^\tp)+b_i b_i^\tp)$. We can use a similar argument. We have
\begin{align*}
&\mathbb{E}(\xi^{k+1} (\xi^{k+1})^\tp\vert \mathcal{F}_{k-1})\\=&\sum_{i=1}^n \left(p_i ((I+\alpha A_i) \xi^k+\alpha b_i)((I+\alpha A_i) \xi^k+\alpha b_i)^\tp\right)\\
=&\sum_{i=1}^n p_i (I+\alpha A_i)\xi^k(\xi^k)^\tp (I+\alpha A_i)^\tp+\sum_{i=1}^n \alpha p_i b_i (\xi^k)^\tp (I+\alpha A_i)^\tp+\sum_{i=1}^n \alpha p_i(I+\alpha A_i)\xi^k b_i^\tp+\alpha^2\sum_{i=1}^n p_i b_i b_i^\tp.
\end{align*}
Taking full expectation and applying the fact $\sum_{i=1}^n p_i b_i=0$ leads to
\begin{align*}
\mathbb{Q}^{k+1}&=\sum_{i=1}^n p_i(I+\alpha A_i) \mathbb{Q}^k (I+\alpha A_i)^\tp+\alpha^2 \sum_{i=1}^n p_i\left(b_i (\mu^k)^\tp A_i^\tp+A_i \mu^k b_i^\tp+b_i b_i^\tp\right)\\
&=\mathbb{Q}^k+\alpha(\bar{A} \mathbb{Q}^k+\mathbb{Q}^k \bar{A}^\tp)+\alpha^2\sum_{i=1}^n p_i \left(A_i \mathbb{Q}^k A_i^\tp +2\sym(A_i \mu^k  b_i^\tp)+b_i b_i^\tp\right).
\end{align*}
This proves the recursive formula for $\mathbb{Q}^k$. Now we can apply the vectorization operation to this formula. For any matrices $A$, $X$, and $B$, we have $\vect(AXB)=(B^\tp \otimes A)\vect(X)$. Hence we can directly show
\begin{align*}
\vect(\mathbb{Q}^{k+1})=&\vect(\mathbb{Q}^k)+\alpha(\vect(\bar{A} \mathbb{Q}^k)+\vect(\mathbb{Q}^k \bar{A}^\tp))+\alpha^2 \sum_{i=1}^n p_i \vect\left(A_i \mathbb{Q}^k A_i^\tp +2\sym(A_i \mu^k  b_i^\tp)+b_i b_i^\tp\right)\\
=&\vect(\mathbb{Q}^k)+\alpha (I\otimes \bar{A}+\bar{A}\otimes I) \vect(\mathbb{Q}^k)+\alpha^2 \left(\sum_{i=1}^n p_i A_i\otimes A_i\right)\vect(\mathbb{Q}^k)\\&+\alpha^2 \left(\sum_{i=1}^n p_i(b_i\otimes A_i+A_i\otimes b_i)\right)\mu^k+\alpha^2\sum_{i=1}^n p_i (b_i\otimes b_i).
\end{align*}
Therefore, we have $\vect(\mathbb{Q}^{k+1})=\mathcal{H}_{22} \vect(Q^k)+\mathcal{H}_{21} \mu^k+\alpha^2 \sum_{i=1}^n p_i(b_i\otimes b_i)$ where $\mathcal{H}_{21}$ and $\mathcal{H}_{22}$ are determined by \eqref{eq:HIID}. Putting this together with $\mu^{k+1}=(I+\alpha \bar{A})\mu^k$ gives us the LTI model in \eqref{eq:LTIIID}. Then notice we have
\begin{align*}
\left(\bmat{\mathcal{H}_{11} & 0 \\ \mathcal{H}_{21} & \mathcal{H}_{22}}\right)^t \bmat{0\\ \alpha^2\sum_{i=1}^n p_i(b_i\otimes b_i)}=\alpha^2\bmat{0\\  (\mathcal{H}_{22})^t (\sum_{i=1}^n p_i(b_i\otimes b_i))}.
\end{align*}

Recall that we have $\mathcal{H}=\bmat{\mathcal{H}_{11} & 0 \\ \mathcal{H}_{21} & \mathcal{H}_{22}}$.
We can apply Proposition 3.6 in \cite{costa2006} to show that $\mathcal{H}$ is Schur stable if and only if $\mathcal{H}_{22}$ is Schur stable. Therefore,
a direct application of Proposition \ref{prop:lti} will lead to \eqref{eq:resIIDF}.
This completes the proof for Theorem \ref{thm:IID}.

\subsection{A detailed proof for Theorem \ref{thm:Markov}}
One may prove this theorem as a corollary of Proposition 3.35 in \cite{costa2006}. For completeness, we add more detailed calculations and present the proof in a self-contained manner.
One can update $q_j^{k+1}$ as
\begin{align*}
q_j^{k+1}&=\sum_{i=1}^n \mathbb{E}\left((H(z^k)\xi^k+\alpha b(z^k)) \textbf{1}_{\{z^k=i\}} \textbf{1}_{\{z^{k+1}=j\}}\right)\\
&=\sum_{i=1}^n \left(H_i\mathbb{E}(\xi^k \textbf{1}_{\{z^k=i\}})\mathbb{P}(\textbf{1}_{\{z^{k+1}=j\}}\vert \xi^k \textbf{1}_{\{z^k=i\}} )+\alpha\mathbb{E}(b(z^k) \textbf{1}_{\{z^k=i\}})\mathbb{P}(\textbf{1}_{\{z^{k+1}=j\}}\vert \xi^k \textbf{1}_{\{z^k=i\}} )\right)\\
&=\sum_{i=1}^n p_{ij}\left( (I+\alpha A_i)q_i^k+\alpha p_i^k b_i \right).
\end{align*}
This leads to the following update rule for $q^k$:
\begin{align*}
\bmat{q_1^{k+1} \\ \vdots \\ q_n^{k+1}}=\bmat{p_{11} (I+\alpha A_1) & \ldots & p_{n1}(I+\alpha A_n)\\ \vdots & \ddots & \vdots\\ p_{1n} (I+\alpha A_1) & \ldots & p_{nn} (I+\alpha A_n)}\bmat{q_1^k \\ \vdots \\ q_n^k}+\alpha\bmat{p_{11}p_1^k I & \ldots & p_{n1} p_n^k I\\ \vdots & \ddots & \vdots \\p_{1n}p_1^k I & \ldots & p_{nn} p_n^k I }\bmat{b_1 \\ \vdots \\ b_n},
\end{align*}
which can be compactly written as $q^{k+1}=(P^\tp \otimes I)\diag(I+\alpha A_i) q^k+\alpha((P^\tp\diag(p_i^k))\otimes I_{n_\xi})b$, where $b$ is the augmented vector
\begin{align*}
b=\bmat{b_1 \\ b_2 \\ \vdots \\ b_n}.
\end{align*}
This proves $q^{k+1}=\mathcal{H}_{11} q^k+\alpha((P^\tp\diag(p_i^k))\otimes I_{n_\xi})b$, where $\mathcal{H}_{11}$ is given by \eqref{eq:HMarkov}.

Next, we perform similar steps to obtain the iterative formula for $Q^k$. One can update $Q_j^{k+1}$ as
\begin{align*}
Q_j^{k+1}&=\sum_{i=1}^n \mathbb{E}\left((H(z^k)\xi^k+\alpha b(z^k)) (H(z^k)\xi^k+\alpha b(z^k))^\tp\textbf{1}_{\{z^k=i\}} \textbf{1}_{\{z^{k+1}=j\}}\right)\\
&=\sum_{i=1}^n \left(H_i\mathbb{E}(\xi^k (\xi^k)^\tp \textbf{1}_{\{z^k=i\}}) H_i^\tp\mathbb{P}(\textbf{1}_{\{z^{k+1}=j\}}\vert \textbf{1}_{\{z^k=i\}} )\right)+\alpha\sum_{i=1}^n \left(H_i\mathbb{E}(\xi^k b(z^k)^\tp \textbf{1}_{\{z^k=i\}}) \mathbb{P}(\textbf{1}_{\{z^{k+1}=j\}}\vert \textbf{1}_{\{z^k=i\}} )\right)\\
&+\alpha\sum_{i=1}^n \left(\mathbb{E}(b(z^k) (\xi^k)^\tp \textbf{1}_{\{z^k=i\}})H_i^\tp \mathbb{P}(\textbf{1}_{\{z^{k+1}=j\}}\vert \textbf{1}_{\{z^k=i\}} )\right)+\alpha^2 \sum_{i=1}^j\mathbb{E}(b(z^k) b(z^k)^\tp\textbf{1}_{\{z^k=i\}})\mathbb{P}(\textbf{1}_{\{z^{k+1}=j\}}\vert  \textbf{1}_{\{z^k=i\}} )\\
&=\sum_{i=1}^n p_{ij}\left( H_i Q_i^k H_i^\tp+2\alpha \sym (H_i q_i^k b_i^\tp)+\alpha^2 p_i^k b_i b_i^\tp\right).
\end{align*}
Now we can apply the vectorization operation to obtain the following equation
\begin{align*}
\vect(Q_j^{k+1})=\sum_{i=1}^n p_{ij} \left((H_i\otimes H_i) \vect(Q_i^k)+\alpha (b_i\otimes H_i+H_i\otimes b_i) q_i^k+\alpha^2 p_i^k b_i\otimes b_i\right),
\end{align*}
which is equivalent to
\begin{align}
\begin{split}
\bmat{\vect(Q_1^{k+1}) \\ \vdots \\ \vect(Q_n^{k+1})}&=\bmat{p_{11} H_1\otimes H_1 & \ldots & p_{n1} H_n\otimes H_n\\ \vdots & \ddots & \vdots\\ p_{1n} H_1\otimes H_1 & \ldots & p_{nn} H_n\otimes H_n}\bmat{\vect(Q_1^k) \\ \vdots \\ \vect(Q_n^k)}+\alpha^2\bmat{p_{11}p_1^k I_{n_\xi} & \ldots & p_{n1} p_n^k I_{n_\xi^2}\\ \vdots & \ddots & \vdots \\p_{1n}p_1^k I_{n_\xi^2} & \ldots & p_{nn} p_n^k I_{n_\xi^2} }\bmat{b_1\otimes b_1 \\ \vdots \\ b_n\otimes b_n}\\
&+\alpha\bmat{p_{11} (b_1\otimes H_1+H_1\otimes b_1) & \ldots & p_{n1} (b_n\otimes H_n+H_n\otimes b_n)\\ \vdots & \ddots & \vdots \\p_{1n} (b_1\otimes H_1+H_1\otimes b_1) & \ldots & p_{nn} (b_n\otimes H_n+H_n\otimes b_n) }\bmat{q_1^k \\ \vdots \\ q_n^k}.
\end{split}
\end{align}
We can compactly rewrite the above equation as $\vect(Q^{k+1})=\mathcal{H}_{22} \vect(Q^k)+\mathcal{H}_{21} q^k+\alpha^2 \diag(p_i^k)(P^\tp \otimes I_{n_\xi^2})\hat{B}$, where $\mathcal{H}_{22}$ and $\mathcal{H}_{21}$ are given by \eqref{eq:HMarkov}, and  $\hat{B}=\bmat{(b_1\otimes b_1)^\tp & \ldots & (b_n\otimes b_n)^\tp}^\tp$.
Putting the recursion formulas for $q^k$ and $\vect(Q^k)$ together leads to the desired state-space model \eqref{eq:MarkovSys}. The rest of the theorem statement follows from direct applications of Equation \eqref{eq:res1}.

\section{Details for perturbation analysis under the Markov assumption}

The perturbation analysis in Section \ref{sec:Markov} relies on a few technical lemmas from matrix perturbation theory. We provide more details here.
We will use the following fact.
\begin{prop}
\label{prop:pert}
Suppose $\lambda$ is a simple eigenvalue of $A$ with left eigenvector $y$ and right eigenvector $x$. Suppose $B$ and $A\otimes I_m$ have the same dimension.
Let $c$ be an eigenvalue of the $m\times m$ matrix $(y \otimes I_m) B (x\otimes I_m)$. Then the matrix $(A\otimes I_m)+\alpha B$ has an eigenvalue yielding the first-order expansion $\lambda+ c \alpha +O(\alpha^2)$ for small $\alpha$.
\end{prop}
The above proposition is a special case of Theorem 2.1 in \cite{moro1997lidskii}. See the remark placed after Theorem 2.1 in \cite{moro1997lidskii} for explanations.
Now we can directly apply the above proposition to analyze the spectral radius of $\mathcal{H}_{11}$ and $\mathcal{H}_{22}$. First recall that
$\mathcal{H}_{11}=(P^\tp \otimes I_{n_\xi}) \diag(I_{n_\xi}+\alpha A_i)=P^\tp \otimes  I_{n_\xi}+\alpha (P^\tp\otimes I_{n_\xi})\diag(A_i)$.
Based on the ergodicity assumption on $\{z^k\}$, $1$ is a simple eigenvalue of $P^\tp$ with left eigenvector $y=\bmat{1 & 1 & \ldots & 1}$ and right eigenvector $p^\infty$ which is  the unique stationary distribution of $\{z^k\}$. 
Since we have $(y\otimes I_{n_\xi}) (P^\tp \otimes I_{n_\xi})\diag(A_i) (p^\infty \otimes I_{n_\xi})=\sum_{i=1}^n p_i^\infty A_i=\bar{A}$,
we can directly apply the above proposition to show
\begin{align}
\lambda_{\max}(\mathcal{H}_{11})=1+\lambda_{\max\real}(\bar{A})\alpha+O(\alpha^2).
\end{align}
Therefore, we have
\begin{align*}
\sigma(\mathcal{H}_{11})=\sqrt{(1+\alpha\real(\lambda_{\max\real}(\bar{A})))^2+(\imag(\lambda_{\max\real}(\bar{A})))^2}\approx 1+\real(\lambda_{\max\real}(\bar{A}))\alpha+O(\alpha^2).
\end{align*}

Next, we do a similar perturbation analysis to show $\mathcal{H}_{22}=1+2\real(\lambda_{\max\real}(\bar{A}))\alpha+O(\alpha^2)$.
Recall that we have 
\begin{align*}
\mathcal{H}_{22}&= (P^\tp \otimes I_{n_\xi^2}) \diag((I_{n_\xi}+\alpha A_i)\otimes (I_{n_\xi}+\alpha A_i))\\&=P^\tp \otimes I_{n_\xi^2}+\alpha ((P^\tp\otimes I_{n_\xi^2})\diag(A_i\otimes I_{n_\xi}+I_{n_\xi}\otimes A_i))+O(\alpha^2).
\end{align*}
Since we have $(y\otimes I_{n_\xi^2}) (P^\tp\otimes I_{n_\xi^2})\diag(A_i\otimes I_{n_\xi}+I_{n_\xi}\otimes A_i) (p^\infty \otimes I_{n_\xi^2})=\bar{A}\otimes I_{n_\xi}+I_{n_\xi}\otimes \bar{A}$,
we can directly apply the above matrix perturbation proposition to show
\begin{align}
\lambda_{\max}(\mathcal{H}_{22})=1+2\lambda_{\max\real}(\bar{A})\alpha+O(\alpha^2),
\end{align}
which leads to the desired first order expansion of $\sigma(\mathcal{H}_{22})$.
Another fact that we used in the above argument is that all the eigenvalues of $\bar{A}\otimes I_{n_\xi}+I_{n_\xi}\otimes \bar{A}$ are in the form of a sum of two eigenvalues of $\bar{A}$.

\paragraph{A remark on the IID case.}
For the IID case, we have $\mathcal{H}_{11}=I+\alpha \bar{A}$. Hence the eigenvalues of $\mathcal{H}_{22}$ are in the form of $1+\alpha\lambda$ where $\lambda$ is an eigenvalue of $\bar{A}$. So there is no need to even perform a perturbation analysis here. We directly have 
\begin{align*}
\sigma(\mathcal{H}_{11})=\sqrt{(1+\alpha\real(\lambda_{\max\real}(\bar{A})))^2+(\imag(\lambda_{\max\real}(\bar{A})))^2}\approx 1+\real(\lambda_{\max\real}(\bar{A}))\alpha+O(\alpha^2).
\end{align*}
To analyze $\sigma(\mathcal{H}_{22})$, first recall that we have $\mathcal{H}_{22}=I_{n_\xi^2}+\alpha (I\otimes \bar{A}+\bar{A}\otimes I)+\alpha^2 \sum_{i=1}^n p_i( A_i\otimes A_i)$. If we assume $\lambda_{\max}(I_{n_\xi^2}+\alpha (I\otimes \bar{A}+\bar{A}\otimes I))$ is a semisimple eigenvalue, then we can apply Proposition \ref{prop:pert} to obtain $\lambda_{\max}(\mathcal{H}_{22})=1+2\lambda_{\max\real}(\bar{A})+O(\alpha^2)$.

\section{Connections to existing finite sample bounds on mean square errors}

Most existing finite sample bounds for TD learning with a constant learning rate have the following form:
\begin{align}
\label{eq:sb1}
\mathbb{E} \norm{\xi^k}^2 \le C_0 \rho^{2k} +C_1,
\end{align}
where $C_0$ is a constant, $\rho^2$ is the convergence rate, and $C_1$ quantifies the final error level. Typically one proves $\rho^2=1-c\alpha+O(\alpha^2)$ for some $c$, and $C_1=O(\alpha)$. One most relevant result of this nature for the Markov noise model was presented as Theorem 7 in \cite{srikant2019finite}. 
Our result justifies the tightness of the result in \cite{srikant2019finite} for the following reasons. 
\begin{itemize}
\item In \cite{srikant2019finite}, the constant $c$ in the rate $\rho^2$ is a constant determined by $\bar{A}$. Using our perturbation analysis, we can see eventually $c$ is going to be determined by the real part of $\lambda_{\max\real}(\bar{A})$. Actually one could modify the argument in \cite{srikant2019finite} to match the constant $c$ with our perturbation analysis result by choosing a slightly better Lyapunov function based on the solution of an linear matrix inequality $\bar{A}^\tp V+V \bar{A}+2\rho V\preceq 0$.
\item In \cite{srikant2019finite}, the constant $C_1$ is at the order of $O(\alpha)$ which matches the perturbation analysis result obtained in our paper. It is possible that the constant $C_1$ may be improved to match the steady state mean square error $\lim_{k\rightarrow \infty}\trace(\mathbb{Q}^k)$ obtained by our perturbation analysis, although we have not pursued such an analysis.
\item In \cite{srikant2019finite}, the rate $\rho$ does not depend on the mixing time property. This is consistent with our theory. Based on our theory, as $\alpha$ gets smaller, the rate $\rho$ becomes independent of the mixing rate $\tilde{\rho}$, although the constant $C_0$ still has some dependence on $\tilde{\rho}$.
\item Our results further show that $\mathbb{E} \norm{\xi^k}^2$ converges to an exact limit at a linear rate. Hence both upper and lower bounds for the mean square TD error are simultaneously provided.
\end{itemize}

It is worth mentioning that the bounds in the form of \eqref{eq:sb1} capture the behaviors of TD learning quite well for small $\alpha$, but can be conservative for large $\alpha$. Our formulas are exact for all $\alpha$. The generality comes at the price of loosing some interpretability for the large learning rate region. How to interpret $\sigma(\mathcal{H}_{22})$ for larger $\alpha$ in a better way remains unclear at this moment.

\section{More discussions on jump system formulations for variants of TD(0)}

Finally, we present some extra details for the jump system formulations of several TD learning algorithms other than TD(0). 
Specifically, all the methods that can be analyzed using the ODE method has the form $\xi^{k+1}=(I+\alpha A(z^k)) \xi^k+\alpha b(z^k)$.
Then taking expectation of $A(z^k)$ and $b(z^k)$ under the stationary distribution and making $\alpha$ arbitrarily small leads to the ODE $\dot{\xi}=\bar{A} \xi$. As commented in Section \ref{sec:JumpFor}, the linear stochastic approximation scheme $\xi^{k+1}=(I+\alpha A(z^k)) \xi^k+\alpha b(z^k)$ is just a MJLS. Now we give detailed references for this type of formulations for various TD learning algorithms.  The detailed linear stochastic approximation form for GTD is given in Section 4 of \cite{sutton2008convergent}.
The detailed linear stochastic approximation form for GTD2 is given in Section 5 of \cite{sutton2009fast}.  TDC yields a similar formulation.
The double temporal difference (DTD) learning method and the average temporal difference (ATD) learning method are proposed in \cite{Niao2019}.  The ODE formulations for both DTD and ATD are presented in the supplementary material of \cite{Niao2019}, yielding straightforward jump system formulations.

It is also possible to model two time-scale methods or off-policy TD learning using the general jump system model \eqref{eq:jump1}. One needs to slightly modify the definitions of $\{H_i\}$ and $\{G_i\}$. Then one can immediately apply the LTI model \eqref{eq:jumpGeneral} to obtain closed-form formulas for the mean square error of these methods.

\end{document}